\documentclass[preprint]{elsarticle}
\usepackage[a4paper, margin=1in]{geometry}

\usepackage{lineno,hyperref}
\modulolinenumbers[10]

\usepackage{amsmath,mathtools,dsfont,amssymb}
\usepackage{todonotes}
\usepackage{bm}
\usepackage{natbib}
\usepackage{xcolor}
\usepackage{amsthm}

\makeatletter
\newcommand{\leqnomode}{\tagsleft@true\let\veqno\@@leqno}
\newcommand{\reqnomode}{\tagsleft@false\let\veqno\@@eqno}
\makeatother

\newtheorem{theorem}{Theorem}
\newtheorem{lemma}[theorem]{Lemma}

\newtheorem{corollary}{Corollary}
\newcommand{\comm}[1]{}

\numberwithin{equation}{section}

\newcommand{\ignore}[1]{ }

\begin{document}

\begin{frontmatter}

\title{A Concentration Bound for LSPE($\lambda$)}

\author[SC]{Siddharth Chandak}
\address[SC]{Department of Electrical Engineering, Stanford University, Stanford, CA 94305, USA}
\ead{chandaks@stanford.edu}

\author[VSB]{Vivek S Borkar\corref{cor1}\fnref{fn1}}
\ead{borkar.vs@gmail.com}
\address[VSB]{Department of Electrical Engineering,  Indian Institute of Technology Bombay,  Powai, Mumbai-400076, India} 

\author[VSB]{Harsh Dolhare}
\ead{harshdolhare99@gmail.com}

\fntext[fn1]{ VSB is supported by the S.\ S.\ Bhatnagar Fellowship from the Council of Scientific and Industrial Research, Government of India.}
\cortext[cor1]{Corresponding author}

\begin{abstract}
The popular LSPE($\lambda$) algorithm for policy evaluation is revisited to derive a concentration bound that gives high probability performance guarantees from some time on. 
\end{abstract}


\end{frontmatter}

\linenumbers

\section{Introduction}
The LSPE($\lambda$) is a family of policy evaluation schemes parametrized by a parameter $\lambda \geq 0$. It typically gives faster convergence in terms of number of iterates at the expense of a somewhat higher per iterate computational budget. Its convergence properties have been extensively analyzed in \cite{LSPEBert}, \cite{LSPEBertBor}, \cite{YuHu2009}. The aim of this article is to give concentration bounds for the iterates to be in a prescribed neighborhood of the goal with a prescribed high probability, from some time $n_0$ on. The choice of this $n_0$ is dictated by how small the decreasing stepsize is required to be for the result to hold. This is in contrast to the usual finite sample analysis that gives a high probability or moment bound for this discrepancy after $n$ steps, starting from time zero, and can be viewed as playing a complementary role. 

The literature on concentration inequalities for reinforcement learning is building up rapidly. Some representative works are \citep{Bhandari, Chen, Chen2, Giannakis,Gugan3,Gugan2,  Li-QL1, Li-QL2, Mansour, Prashanth, Qu, Sidford, Srikant, Wain1, Wain2}, which deal with finite time bounds for  Q-learning and/or TD$(\lambda)$ algorithms. 
 There are also  works which seek bounds, either finite time or asymptotic (e.g., in terms of regret) on the difference between the value function under the learned policy and the optimal value function \citep{Chi_Jin, Yang, Lin_Yang}.
 Our result is in the  spirit of a different strand in stochastic approximation theory (of which most reinforcement learning algorithms are special instances), which aims for  a concentration bound for the iterates \textit{from some time on}, to be precise, `\textit{for all $n \geq n_0$ for a suitably chosen $n_0$}' \citep{Borkar0, Borkar-conc, Chandak, Gugan, Kamal}.  See \citep{Borkar-conc, Chandak, BorkarQ} for  applications to reinforcement learning. In addition to this difference in formulation, this work also differs significantly from these works in the flavor of the proofs, in particular in its use of a recent concentration inequality from \cite{Paulin}. See \cite{Bucsoniu,  lazaric12} for further background on least squares algorithms. 

The article is organized as follows. The next section gives the background on the LSPE$(\lambda)$ algorithm and states the main result.  The proof of the main result follows in section \ref{sec:proof}, which develops the argument through a sequence of lemmas that are separately listed without proof in subsection \ref{subsec:lemmata} for ease of reading. The rather technical proofs follow in \ref{app:proofs}. Section \ref{sec:conv} presents the proof for the almost sure convergence of our algorithm. While the convergence is a known fact, this alternative argument has embedded in it a kind of regret bound, though possibly not the best one. The concluding section \ref{sec:conc} gives some pointers for future work. Appendix A states some standard facts used in the proof, such as  the key concentration inequality from \cite{Paulin} used in the proof and some linear algebraic facts. Appendix B, as mentioned above, details the more technical parts of the proof of the main result.

\section{Background and Main Result}\label{sec:main}
The LSPE($\lambda$) algorithm is for policy evaluation, i.e.,  we fix the policy a priori and do not optimize over policies. This gives us a time-homogeneous uncontrolled Markov chain $\{X_n\}$ with transition probabilities $p(\cdot| \cdot)$. Let $S$ be the state space of this chain and $s =|S|$. We assume that the chain is irreducible with a unique stationary distribution denoted by $ \pi = [\pi(1),\pi(2),\dots,\pi(s)]$. We denote by $D$ the $s \times s$ diagonal matrix whose $i^{th}$ diagonal entry is $\pi(i)$. When a transition from state $i$ takes place, we incur a cost $k(i)$. We want to evaluate the long term discounted cost for each state given by the `value function'
$$
V(i) = \mathbb{E}\left[\sum_{m=0}^{\infty}\alpha^m k(X_m)\Big|X_0 = i\right], i \in S.
$$
Here $\alpha$ is a discount factor with $0<\alpha<1$. Using `one step analysis', $V(\cdot)$ is seen to satisfy the linear Bellman  equation
$$V(i) = k(i) + \alpha\sum_{j\in S}p(j|i)V(j), i \in S.$$
This suffers from the `curse of dimensionality' when $|S| >> 1$. To alleviate it, we may approximate $V$ as a linear combination of $M$ prescribed basis function $\phi_i:S\to \mathbb{R}$, $1\leq i\leq M$. Thus, $V(i) \approx \sum_{j=1}^{M}r_j\phi_j(i)$, written in vector form as $V = \Phi r$ where $ r = [r_1,r_2,\dots,r_M]^T$ and $\Phi$ is a $s \times M$ matrix whose $i^{th}$ column is $\phi_i$. The approximate Bellman equation then is
$$\Phi r \approx k+\alpha P\Phi r,$$
where $P \coloneqq [[p(j|i)]]$ is the transition matrix. 

The LSPE($\lambda$), as described in \cite{LSPEBert}, iteratively updates the weight vector $r_n$ to minimize the squared norm error between the right and left hand sides of the approximate Bellman equation. Let $\varphi(i)$ denote the $i$th row of $\Phi$, transposed to make it a column vector. The update can be written as
\begin{equation}\label{LSPE-update}
 r_{n+1}=r_n+a(n)B_n^{-1}(A_nr_n+b_n), 
\end{equation}
where, for a small $\varrho > 0$ and $I :=$ the identity matrix,
\begin{equation}\nonumber
    B_n = \frac{1}{n+1}\left(\varrho I + \sum_{m=0}^{n}\varphi(X_m)\varphi(X_m)^T\right),
\end{equation}
\begin{equation}\label{def:A_n-and-B_n}
    A_n = \frac{1}{n+1}\sum_{m=0}^{n}z_m(\alpha\varphi(X_{m+1})-\varphi(X_m))^T,
\end{equation}

\begin{equation}\nonumber
    b_n = \frac{1}{n+1}\sum_{m=0}^{n}z_mk(X_m), 
\end{equation}
\begin{equation}\label{def:b_n-and-z_m}
\mbox{where} \ \    z_m = \sum_{t=0}^{m} (\alpha \lambda)^{m-t}\varphi(X_t).
\end{equation}

Observe that our definitions of $A_n, B_n, b_n$ are normalized versions of those in \cite{LSPEBert}.  $\{a(n)\}$ denote the stepsizes which are assumed to satisfy the usual Robbins-Monro conditions $\sum_na(n) = \infty, \sum_na(n)^2<\infty$. We additionally assume that $a(n)<1$ for all $n$.  It should be mentioned here that $((n+1)B_n)^{-1}$ is usually computed recursively using the Sherman-Morrison-Woodbury formula \cite{loan_charles}. This recursion is initiated with initial condition given by a small constant $\varrho$ times the identity matrix $I$ to ensure invertibility. The error caused by this choice gets washed away asymptotically because of convergence of the iterates and averaging thereof.

It is shown in \cite{LSPEBert} that LSPE($\lambda$) converges to $r^{*}$ defined as the solution to the linear system
\begin{equation}\label{def:r^{*}}
    Ar^{*}+b =0,
\end{equation}
where $A$ and $b$ are given by
\begin{equation}\nonumber
    A = \Phi^TD(\alpha P-I)\sum_{j=0}^{\infty} (\alpha \lambda P)^j\Phi,
\end{equation}

\begin{equation}\label{def:Aandb}
    b = \Phi^TD\sum_{j=0}^{\infty} (\alpha \lambda P)^j\Bar{k}.
\end{equation}
for $\Bar{k}(i) \coloneqq$ the vectorized $k(\cdot)$ and $P = [[p(j|i)]]_{i,j }$.

Let $B = \Phi^TD\Phi$ and define $N = I + B^{-1}A$. It is proved in \cite{LSPEBertBor} that magnitude of all eigenvalues of N is less than 1. 
{\color{black}
Define $H\in\mathbb{R}^{M\times M}$ as the unique positive definite solution to the Lyapunov equation
$$N^THN-H=-I.$$
Then the matrix $N$ is a contraction w.r.t.\ the norm induced by $H$, i.e., $\|Nx\|_H\leq \beta\|x\|_H$, where $\beta \in (0,1)$ is the contraction factor. The existence of such a matrix $H$ and the contractivity are shown in Lemma \ref{lemma:H-contra} in \ref{app:prelims}.
}

{\color{black}
We make an additional assumption on the stepsize sequence $a(n)$, viz., 
$$\frac{\mu_1}{n}\leq a(n) \leq \mu_3 \left(\frac{1}{n}\right)^{\mu_2}$$ 
for all $n\geq 1$, where $\mu_1,\mu_3>0$ and $\frac{1}{2}+\theta<\mu_2\leq 1$. Here $\theta\in(0,1/2)$. For $n\geq m  \geq 0$ and $\epsilon>0$, we define:
\begin{eqnarray*}
b_{m}(n) &\coloneqq&\sum_{k=m}^{n}a(k), \\
\xi_n(\epsilon) &\coloneqq& \epsilon (n-1)^{1/2+\theta-\mu_2}+(n-1)^{-\mu_2}.
\end{eqnarray*}
Note that for a fixed $\epsilon$, $\xi_n(\epsilon)\downarrow 0$ as $n\uparrow \infty$.}
Let $\lambda_{max}(\cdots)$, $\lambda_{min}(\cdots)$ and $\sigma_{max}(\cdots)$ denote the largest eigenvalue, the smallest eigenvalue, and the largest singular value of the matrix `$\cdots$', respectively. Further, let us denote by $k_{max} = \|k\|_\infty$ and $\Phi_{max} \coloneqq \|\Phi\|_\infty$. For any matrix $A$, we define the matrix norm $\|A\|$ as the operator norm, i.e., $\|A\|=\sup_{\|x\|=1}\|Ax\|$. 

We now state our main result.
{\color{black}
\begin{theorem}\label{main-thm}
Let $n_0, \epsilon>0$ and $0<\delta<1$. Then there exist finite positive constants $K_1,\ldots,K_4$, such that for $n_0$ large enough and $\delta$ small enough to satisfy  $\beta +\delta + \frac{K_1}{n_0+1} < 1$ and $n \geq n_0$, the inequality
    \begin{align*}
   \|r_n-r^*\|_H \leq e^{-(1-\beta)b_{n_0}(n-1)}\|r_{n_0}-r^{*}\|_{H} + \frac{K_3\left(\|r_{n_0}\|_H+1\right)\xi_n(\epsilon)}{1-(\beta + \delta +  \frac{K_1}{n_0+1})},
\end{align*}
holds with probability exceeding 
\begin{equation*}
1-8M^2\sum_{k=n_0}^{n-1}\left(\exp\left(\frac{-(k+1)\delta^2}{M^2K_2}\right) + \exp\left(\frac{-(k+1)^{2\theta}\epsilon^2}{M^2K_4}\right) \right).
\end{equation*}
\noindent In particular,
    \begin{align}\label{main-thm-bound}
   \|r_n-r^*\|_H \leq e^{-(1-\beta)b_{n_0}(n-1)}\|r_{n_0}-r^{*}\|_{H} + \frac{K_3\left(\|r_{n_0}\|_H+1\right)\xi_n(\epsilon)}{1-(\beta + \delta +  \frac{K_1}{n_0+1})},
\end{align}
$\forall \; n\geq n_0$ with probability exceeding
\begin{equation}\label{main-thm-prob}
1-8M^2\sum_{k=n_0}^{\infty}\left(\exp\left(\frac{-(k+1)\delta^2}{M^2K_2}\right) + \exp\left(\frac{-(k+1)^{2\theta}\epsilon^2}{M^2K_4}\right) \right).
\end{equation}
\end{theorem}
}

We conclude this section with some observations concerning the above theorem and its proof that follows.
\begin{enumerate}

\item  Note that our result is not a traditional sample complexity result that gives a high probability bound at some time $n$ sufficiently large starting from time zero, but a `tail' bound that is a high probability bound for $n \geq n_0$, for \textit{every} $n_0$ from a prescribed value onwards. Admittedly, the bound depends on $r_{n_0}$, but this is a.s.\ bounded because the iteration is a.s.\ convergent and a basic bound on its norm can be easily obtained using the discrete Gronwall inequality. A better result is to stitch our result with a classical sample complexity bound to get a `\textit{for all time}' result, which has been done for other algorithms in our earlier works \cite{Borkar-conc,Chandak}. Unfortunately a classical sample complexity bound seems unavailable for LSPE($\lambda$).

\item  \color{black} $\|r_{n_0}\|_H$ has a bound depending on $\|r_0\|_H$ which can be derived using discrete Gronwall's inequality. As an example, consider a stepsize of the form $a(n)=c/n$ for some $c>0$. Then for a large enough $\varrho$, we have $\|r_{n_0}\|_H = O(n_0^c)$. Even though our bound on $\|r_{n_0}\|_H$ increases with $n_0$, the bound on $\|r_n-r^*\|_H$ goes to $0$ as $n\uparrow \infty$.
\end{enumerate}
\section{Proof of the Main Result}\label{sec:proof}
{\color{black}
This section is structured as follows: we first present the proof of the main theorem. We use two lemmas for the proof which have been stated at the end of this section as they require notation defined in the proof.
}
\begin{proof}[Proof of Theorem \ref{main-thm}]
Define $y_n$ for $n\geq n_0$ by 
\begin{equation}\label{iteration_y_n}
y_{n+1}=y_n+a(n)B^{-1}\left(Ay_n+b\right),
\end{equation}
where $y_{n_0}=r_{n_0}$. The iteration for LSPE($\lambda$) is given by 
$$r_{n+1}=r_n+a(n)B_n^{-1}(A_nr_n+b_n).$$
Using the definition of $y_n$ from \eqref{iteration_y_n} and by adding and subtracting terms, we get
\begin{align*}
r_{n+1}-y_{n+1}=&\;(1-a(n))(r_n-y_n)\\
& +a(n)\left((I+B^{-1}A)(r_n-y_n)\right)\\
& +a(n)\Big(\left(B_n^{-1}-B^{-1}\right)A_n +B^{-1}\left(A_n-A\right)\Big)(r_n-y_n)\\
&+a(n)\Big(B_n^{-1}\left((A_n-A)y_n+(b_n-b)\right)+\left(B_n^{-1}-B^{-1}\right)(Ay_n+b)\Big).
\end{align*}

Define 
\begin{equation}\label{def:delta}
      \delta(n)\coloneqq(B_n^{-1}-B^{-1})A_n+B^{-1}(A_n-A),
 \end{equation}
 {\color{black}
\begin{equation}\label{def:eps1}
    \varepsilon_1(n)\coloneqq B_n^{-1}\left(A_n-A\right)+ \left(B_n^{-1}-B^{-1}\right)A,
\end{equation}
and 
\begin{equation}\label{def:eps2}
    \varepsilon_2(n)\coloneqq B_n^{-1}(b_n-b)+\left(B_n^{-1}-B^{-1}\right)b.
\end{equation}
Also define $\varepsilon(n)=\varepsilon_1(n)y_n+\varepsilon_2(n)$.}
For $n,m \geq 0$, let $\chi(n,m) = \prod_{k=m}^{n}(1-a(k))$ if $n \geq m$ and 1 otherwise. \textcolor{black}{Note that $r_{n_0}=y_{n_0}$. }Then iterating the above equation for $n_0 \leq m \leq n$, we obtain,
\begin{align*}
r_{m+1}-y_{m+1}=\;& \sum_{k=n_0}^m\chi(m,k+1)a(k)(I+B^{-1}A)(r_k-y_k)\nonumber\\
&+\sum_{k=n_0}^m\chi(m,k+1)a(k)\delta(k)(r_k-y_k)\nonumber\\
&+\sum_{k=n_0}^m\chi(m,k+1)a(k)\varepsilon(k). \nonumber
\end{align*}

By Lemma \ref{lemma:H-contra} of Appendix A,  $I+B^{-1}A$ is a contraction in the $\| \cdot \|_H$ with contraction factor $\beta  \in (0,1)$. Therefore we can write
\begin{align}
\|r_{m+1}-y_{m+1}\|_H \ \leq& \sum_{k=n_0}^m\chi(m,k+1)a(k)\beta\|r_k-y_k\|_H\nonumber\\
&+ \sum_{k=n_0}^m\chi(m,k+1)a(k)\|\delta(k)(r_k-y_k)\|_H\nonumber\\
&+ \sum_{k=n_0}^m\chi(m,k+1)a(k)\|\varepsilon(k)\|_H.\nonumber
\end{align}

{\color{black}

In what follows, the constants $\{K_i\}$ will be defined later. Lemma \ref{lemma:eps-delta}, stated at the end of this section, gives concentration bounds on $\|\epsilon(k)\|_H$ and $\|\delta(k)\|_H$. It shows that $\|\delta(k)\|_H \leq \delta + \frac{K_1}{k+1}$ holds with probability exceeding some $1-p_{1_k}$ where $$p_{1_k}=4M^2\exp\left(\frac{-(k+1)\delta^2}{M^2K_2}\right)$$ for some constants $K_1$ and $K_2$. Similarly it also shows that $\|\varepsilon_1(k)\|_H \leq \frac{\epsilon}{(k+1)^{1/2-\theta}} + \frac{K_5}{k+1}$ and $\|\varepsilon_2(k)\|_H \leq \frac{\epsilon}{(k+1)^{1/2-\theta}} + \frac{K_6}{k+1}$ together hold with probability greater than some $1-p_{2_k}$ where $$p_{2_k}=8M^2\exp\left(\frac{-(k+1)^{2\theta}\epsilon^2}{M^2K_4}\right)$$ for some constants $K_4, K_5$ and $K_6$. Recall that $\theta\in(0,1/2]$. This implies that, for some constant $K_3'$, $$\|\varepsilon(k)\|_H \leq K_3'\left(\frac{\epsilon}{(k+1)^{(1/2-\theta)}} + \frac{1}{k+1}\right)(\|r_{n_0}\|_H+1)$$ holds with probability greater than $1-p_{2_k}$, 
where we use the result from Lemma \ref{bound_y_n}. Then, for some constant $K_3$, 
$$\max_{n_0\leq k\leq n-1}\|\delta(k)\|_H\leq\delta+\frac{K_1}{n_{0}+1}$$ 
and 
\begin{align}\label{sum_of_eps}
\sum_{k=n_0}^m\chi(m,k+1)a(k)\|\varepsilon(k)\|_H \leq K_3\left(\|r_{n_0}\|_H+1\right)\left(\epsilon m^{1/2+\theta-\mu_2}+m^{-\mu_2}\right), \;\forall\; n_0\leq m \leq n-1
\end{align}
with probability exceeding $1-\sum_{k=n_0}^{n-1}(p_{1_k}+p_{2_k})$.  The second statement here uses the assumption that $\frac{\mu_1}{n}\leq a(n)\leq \mu_3\left(\frac{1}{n}\right)^{\mu_2}$ and is shown in Lemma \ref{lemma:eps-delta} (c). We select $\delta$ and $\varepsilon$ such that $\sum_{k=n_0}^{n-1}(p_{1_k}+p_{2_k}) <1 $. The values for these are specified in the statement of Lemma \ref{lemma:eps-delta}. For ease of notation, let $\zeta_m$ denote the right hand side of the inequality in (\ref{sum_of_eps}).
 
Then with probability exceeding  $1 - \sum_{k=n_0}^{n-1}(p_{1_k}+p_{2_k})$,  we have the following:
\begin{align}
\|r_{m+1}-y_{m+1}\|_H\leq \sum_{k=n_0}^m\chi(m,k+1)a(k)\left(\beta+\delta + \ \frac{K_1}{n_{0}+1}\right)\|r_k-y_k\|_H+\zeta_m.
\end{align}
}

Let $\omega_m =\sup_{n_{0}\leq k \leq m} \|r_k-y_k\|_H $ for $n_{0}\leq m \leq n$. Now, note that for any $0<k\leq m$,
$$\chi(m,k)+\chi(m,k+1)a(k)=\chi(m,k+1),$$ and hence
$$\chi(m,n_0)+\sum_{k={n_0}}^m\chi(m,k+1)a(k)=\chi(m,m+1)=1.$$ This implies that 
\begin{equation}\label{Bound_chi_a}
    \sum_{k=n_0}^m\chi(m,k+1)a(k) \ \leq \ 1.
\end{equation}

{\color{black}
Hence for $n_{0}\leq m \leq n$,
\begin{equation*}
\omega_{m+1} \leq 
\left(\beta + \delta + \frac{K_1}{n_{0}+1}\right)\omega_m + \zeta_m.
\end{equation*}

Since $\omega_m\leq \omega_{m+1}$, for all $n_{0}\leq m\leq n$,
\begin{equation*}
\omega_{m+1} \leq \frac{1}{1 - (\beta + \delta + \frac{K_1}{n_{0}+1})}\times
\zeta_m.
\end{equation*}
}
Using the fact that $Ar^*+b=0$, we have
\begin{align*}
y_{n+1}-r^*&=y_n-r^*+a(n)\left(B^{-1}(Ay_n+b)-B^{-1}(Ar^*+b)\right)\\
&= y_n-r^*+a(n)\left(B^{-1}A(y_n-r^*)\right)\\
&=(1-a(n))(y_n-r^*)+a(n)\left(I+B^{-1}A\right)(y_n-r^*).
\end{align*}
This implies that 
\begin{align*}
\|y_{n+1}-r^*\|_H&\leq(1-a(n))\|(y_n-r^*)\|_H+a(n)\beta\|y_n-r^*\|_H\\
&=(1-(1-\beta)a(n))\|y_n-r^*\|_H.
\end{align*}
We then have 
\begin{eqnarray*}
    \|y_n-r^*\|_H&\leq& \psi(n,n_0)\|y_{n_0}-r^*\|_H\\&\leq& e^{-(1-\beta)b_{n_0}(n-1)}\|r_{n_0}-r^*\|_H.
\end{eqnarray*}
Combining this with the fact that $\|r_n-y_n\|_H\leq \omega_n$, gives us the required result of part (a) of Theorem \ref{main-thm}. 

For part (b) of Theorem \ref{main-thm}, we use union bound to get bounds on $\varepsilon(k)$ and $\delta(k)$ for all $k\geq n_0$. 
\end{proof}

{\color{black}
\subsection{Lemmata}\label{subsec:lemmata}
Now we present the lemmas used in the proof above. The proofs for these have been postponed to \ref{app:proofs}. First we present a lemma adapted from \cite{Borkar-conc} which bounds the iterates $\{y_n\}$:
\begin{lemma}\label{bound_y_n}
$\sup_{n\geq n_0} \|y_n\|_H\leq \|y_{n_0}\|_H+\frac{\|B^{-1}b\|_H}{1-\beta}, a.s.$
\end{lemma}

Next we need the concentration bounds for $\|\varepsilon\|_H$ and $\|\delta\|_H$ used in the proof, and which specify the values of $p_{1_k}$ and $p_{2_k}$. Before this, we give high probability bounds on $\|b_n - b\|_H, \|B_n - B\|_H, \|A_n - A\|_H$ using the Markov chain concentration inequality in Theorem \ref{thm:Paulin} in Appendix A.

\begin{lemma}\label{lemma:conc_ineq}
The following holds for the positive definite matrix $H$ defined above,
\begin{enumerate}[(a)]
    \item For $b_n$, $b$ as defined in (\ref{def:b_n-and-z_m}) and (\ref{def:Aandb}) respectively,
    \begin{equation*}
        P\left(\|b_n-b\|_H \leq a+\frac{C_1}{n+1}\right)\geq 1 - 2M\exp\left(\frac{-(n+1)a^2}{M\gamma_1}\right)
    \end{equation*}
    where $a$ is any positive real and $C_1$ and $\gamma_1$ are positive constants.
    \item For $A_n$, $A$  defined in (\ref{def:A_n-and-B_n}) and (\ref{def:Aandb}) resp., 
    \begin{equation*}
    P\left(\|A_n-A\|_H \leq  a+\frac{C_2}{n+1}\right)\geq  1 - 2M^2\exp\left(\frac{-(n+1)a^2}{M^2\gamma_2}\right)
    \end{equation*}
    where $a$ is any positive real and $C_2$ and $\gamma_2$ are positive constants.
    \item For $B_n$, $B$ as defined in (\ref{def:A_n-and-B_n}),
    \begin{equation*}
        P\left(\|B_n-B\|_H \leq  a+\frac{C_3}{n+1}\right)\geq 1 - 2M^2\exp\left(\frac{-(n+1)a^2}{M^2\gamma_3}\right)
    \end{equation*}
    where $a$ is any positive real and $C_3$ and $\gamma_3$ are positive constants. As a consequence of this,
    \begin{eqnarray*}
    P\left(\|B_n^{-1}-B^{-1}\|_H\leq a+ \frac{C_4}{n+1} \right) \geq  1-2M^2\exp\left(\frac{-(n+1)a^2}{M^2(C^{'})^2\|B^{-1}\|_H^2\gamma_3}\right),
    \end{eqnarray*}
   where $a$ is a positive real and $C_4$ is a positive constant. Here $C^{'}$ is a constant such that $\|B_n^{-1}\|_H \leq C^{'}$ for all $n$ (Lemma 4.3 (a) from \cite{LSPEBert}). 
\end{enumerate}
\end{lemma}
Similar to $C^{'}$, we also define $C_5$ as a constant such that $\|A_n\|_H\leq C_5$ for all $n$ (Lemma \ref{lemma:A_n-bound}, \ref{app:proofs}). We now present the final lemma which completes the proof for Theorem \ref{main-thm}. The constants $K_1, K_2,...$ here are specified later.
\begin{lemma}\label{lemma:eps-delta}
For the positive definite $H$ defined above,
\begin{enumerate}[(a)]
    \item For $\delta(n)$ as defined in (\ref{def:delta}),
    \begin{equation*}
        P\left(\|\delta(n)\|_H\leq \delta+ \frac{K_1}{(n+1)}\right)\geq 1 - 4M^2\exp\left(\frac{-(n+1)\delta^2}{M^2K_2}\right)
    \end{equation*}
    where $K_1$ and $K_2$ are positive constants.
\item For $\varepsilon_1(n)$ and $\varepsilon_2(n)$ as defined in (\ref{def:eps1}) and (\ref{def:eps2}), respectively, 
\begin{equation*}
\|\varepsilon_1(n)\|_H\leq \frac{\epsilon}{(n+1)^{1/2-\theta}}+\frac{K_5}{n+1}
\end{equation*}    
and 
\begin{equation*}
    \|\varepsilon_2(n)\|_H\leq \frac{\epsilon}{(n+1)^{1/2-\theta}}+\frac{K_6}{n+1}
\end{equation*}
together hold with probability exceeding $$1-8M^2\exp\left(\frac{-(n+1)^{2\theta}\epsilon^2}{M^2K_4}\right),$$
where $K_5, K_6$ and $K_4$ are positive constants.
\item For $n_0\leq m\leq n-1$,
\begin{align*}
\sum_{k=n_0}^m\chi(m,k+1)a(k)\|\varepsilon(k)\|_H \leq K_3\left(\|r_{n_0}\|_H+1\right)\left(\epsilon m^{1/2+\theta-\mu_2}+m^{-\mu_2}\right)
\end{align*}
holds with probability $1-8M^2\sum_{k=n_0}^m\exp\left(\frac{-(k+1)^{2\theta}\epsilon^2}{M^2K_4}\right)$.
\end{enumerate}
\end{lemma}
}

\section{Convergence}\label{sec:conv}

We show that Theorem \ref{main-thm} recovers the almost sure convergence of the iterates to $r^*$ from \cite{LSPEBert}. 

{\color{black}
\begin{corollary}
    $r_n\to r^*\;a.s.$
\end{corollary}

\begin{proof}
Fix $\delta$ to be a constant which satisfies the assumptions of Theorem \ref{main-thm}. Also fix $\epsilon >0$. Define $\{\hat{b}(m)\}$ to be a sequence such that $0<\hat{b}(m)\downarrow 0$. Then choose $n_0=n_0(m)$ sufficiently large such that the probability in (\ref{main-thm-prob}) exceeds $1-\frac{1}{m^2}$. Pick $n_1(m)>n_0(m)$ large enough such that $$e^{-(1-\beta)b_{n_0(m)}(n_1(m)-1)}\|r_{n_0(m)}-r^{*}\|_{H}\leq \frac{\hat{b}(m)}{2},$$ and 
$$\frac{K_3\left(\|r_{n_0(m)}\|_H+1\right)\xi_{n_1(m)}(\epsilon)}{1-(\beta + \delta +  \frac{K_1}{n_0(m)+1})}\leq \frac{\hat{b}(m)}{2}.$$ This implies that 
$$\sum_m P\left(\sup_{n\geq n_1(m)} \|r_n-r^*\|_H>\hat{b}(m)\right)\leq \sum_m \frac{1}{m^2}< \infty.$$
By the Borel-Cantelli lemma, $\|r_{n_1(m)}-r^*\|_H\leq \hat{b}(m)$ for $m$ sufficiently large, a.s. Since $\hat{b}(m)\downarrow0$, it follows that $r_m\rightarrow r^*$ a.s.
\end{proof}

Thus $\{\hat{b}(m)\}$ also serves as a regret bound for the `cost' $\|r_n-r^*\|_H$, though possibly not the tightest one.
}

\section{Conclusions}\label{sec:conc}

In this work we provided a rigorous high probability concentration bound for the error in LSPE$(\lambda)$ algorithm for all time, from some time $n_0$ onwards, chosen suitably (essentially, so that the decreasing step sizes have become sufficiently small). The proof uses some novel tools including a concentration inequality from \cite{Paulin}. It seems very promising that analogous results would be possible for other algorithms of a similar flavor, such as zap Q-learning \citep{Meyn}.

{\color{black}
\appendix
\section{Some Preliminaries}\label{app:prelims}
We first present the key concentration inequality for Markov chains used throughout the proof.
\begin{theorem}\label{thm:Paulin}\citep[Corollary 2.10]{Paulin}
Let $X\coloneqq(X_0,\cdots,X_n)$ be a Markov chain, taking values in the state space $S$, transition probability matrix $P$ and stationary distribution $\pi$. Recall the total variation distance of two distributions $p$ and $q$ on a measurable space $(\Omega, \mathcal{F})$ given by 
$$
d_{TV}(p,q) \coloneqq \sup_{A\in \mathcal{F}}|p(A)-q(A)|.
$$
Suppose $P^n(x,\cdot)$ denote the distribution of $X_n$ conditioned on $X_0=x$. Then the  mixing time of the Markov chain $\tau(\epsilon)$ for $0 \leq \epsilon < 1$ is defined by:
$$
d(t) \coloneqq \sup_{x \in S} d_{TV}(P^t(x,\cdot),\pi), \;\;\;\;\;\ \tau(\epsilon) \coloneqq \min\{t: d(t)\leq \epsilon\}
$$

Let 
$$\tau_{min}\coloneqq \inf_{0\leq\epsilon<1}\tau(\epsilon)\left(\frac{2-\epsilon}{1-\epsilon}\right)^2.$$

Let $f: S^{n+1}\rightarrow \mathbb{R}$ and $c\coloneqq (c_0,\cdots,c_n)\in\mathbb{R}^{n+1}_{+}$ satisfy the following property: for every $i=0,\cdots,n$ and every $x=(x_0,\cdots,x_n), x'=(x'_0,\cdots,x'_n)\in S^{n+1}$ that differ only in the $i$\textsuperscript{th} coordinate (i.e., $x_j=x'_j$ for all $j\neq i$), we have
$$|f(x)-f(x')|\leq c_i.$$
Then for any $t\geq0$ ,
\begin{equation}
    \mathbb{P}(|f(X)-\mathbb{E}(f(X))|\geq t)\leq2\exp\left(\frac{-2t^2}{\|c\|_2^2\tau_{min}}\right).
\end{equation}
\end{theorem}

The following lemma proves existence of the matrix $H$ above and shows that $N$ is a contraction under the $H$-norm.
\begin{lemma}\label{lemma:H-contra}
For $N=I+B^{-1}A$, there exists a unique positive definite matrix $H\in\mathbb{R}^{M\times M}$ such that $H$ satisfies $N^THN-H=-I$. Also, the matrix $N$ is a contraction under the norm induced by the matrix $H$.
\end{lemma}
\begin{proof}
We know that $N$ has all eigenvalues less than $1$. It is well known (see, e.g., \cite{Kailath}) that a matrix $N \in \mathbb{R}^{M\times M}$ is asymptotically stable (i.e., all eigenvalues lie within the open unit disc) if and only if for every positive definite matrix $R\in \mathbb{R}^{M\times M}$, there exits a unique positive definite matrix $H \in \mathbb{R}^{M\times M}$ satisfying the Lyapunov equation
\begin{equation}\label{Lia}
    N^THN - H = -R.
\end{equation}
For any $x \in \mathbb{R}^{M}$, 
\begin{align*}
&x^TN^THNx = x^THx - x^TRx \nonumber \\
\implies &\|Nx\|^2_H =\|x\|^2_H - x^TRx.
\end{align*}
Here $\|\cdot\|_H$ for a positive definite matrix $H$ denotes the norm $\|x\|_H = \sqrt{x^THx}$. Let $R = I$ in above equation and let $(\lambda_i > 0, v_i)$ denote the eigenvalue-eigenvector pairs of $H$. Then since $H > I$ in the usual order for positive definite matrices (as is obvious from (\ref{Lia}) with $R = I$), we have $\lambda_i \geq 1 \ \forall i$. Hence
\begin{eqnarray*}
\lefteqn{\beta \coloneqq \max_{\|x\| \neq 0}\frac{\|Nx\|_H}{\|x\|_H} = \max_{\|x\| \neq 0}\frac{\sqrt{x^TN^THNx}}{\sqrt{x^THx}}} \\
&=&  \max_{\|x\| \neq 0}\frac{\sqrt{x^T(H - I)x}}{\sqrt{x^THx}} \\
&=& \max_{\|x\| \neq 0}\frac{\sqrt{\sum_i(\lambda_i - 1)\langle x, v_i\rangle^2}}{\sqrt{\sum_i\lambda_i\langle x, v_i\rangle^2}} \ < \ 1.
\end{eqnarray*}
We conclude that $N$ is a contraction w.r.t.\ the norm induced  by $H$, specifically $\|Nx\|_{H}\leq \beta \|x\|_{H} $ for some $\beta\in(0,1)$.
\end{proof}

Next, we present a lemma containing some linear algebraic preliminaries, which are used to bound the norm induced by $H$ given bounds on the infinity norm.
\begin{lemma}\label{lemma:vec-mat-bounds}
For a positive-definite matrix $H$, 
\begin{enumerate}[(a)]  
    \item if $u \in \mathbb{R}^M$ with $\|u\|_\infty \leq a$, then
    $$ \|u\|_{H} \leq \sqrt{\lambda_{max}(H)M}a,$$
    \item if $V \in \mathbb{R}^{M\times M}$ with $\|V\|_\infty \leq a$, then 
$$\|V\|_{H} \leq \sqrt{\frac{\lambda_{max}(H)}{\lambda_{min}(H)}}Ma.$$
\end{enumerate}
\end{lemma}
\begin{proof}
\begin{enumerate}[(a)]  
    \item For any vector $u$,
    $$
    \|u\|_{H}^2 \ = \ u^THu \leq \lambda_{max}(H)\|u\|_{2}^2  \ \leq \ \lambda_{max}(H)Ma^2\nonumber
    $$
    $$
    \implies \ \ \|u\|_{H} \ \leq \ \sqrt{\lambda_{max}(H)M}a.
    $$
    \item For any non-zero vector $x \in \mathbb{R}^M$,
    \begin{eqnarray*}
    \frac{x^TV^THVx}{x^THx} \leq \frac{\lambda_{max}(H)\|Vx\|_{2}^2}{\lambda_{min}(H)\|x\|_2^2} \leq \frac{\lambda_{max}(H)}{\lambda_{min}(H)}\|V\|_{2}^2
\ \leq \ \frac{\lambda_{max}(H)}{\lambda_{min}(H)}\|V\|_{F}^2.
\end{eqnarray*}
In the last step we use the fact that for any matrix $Q$, $\|Q\|_2\leq\|Q\|_F$ where $\|.\|_F$ denotes the Frobenius norm of the matrix. Therefore
$$\|V\|_{H}^2 \leq  \frac{\lambda_{max}(H)}{\lambda_{min}(H)}\|V\|_{F}^2 \leq \frac{\lambda_{max}(H)}{\lambda_{min}(H)}M^2a^2$$
$$
\implies \|V\|_{H} \leq \sqrt{\frac{\lambda_{max}(H)}{\lambda_{min}(H)}}Ma.$$ 
\end{enumerate}
\end{proof}

\section{Technical Proofs}\label{app:proofs}
In this section, we present the proofs for concentration bounds given in Lemma \ref{lemma:conc_ineq} and Lemma \ref{lemma:eps-delta} and also a bound on $\|y_n\|_H$. 
\begin{proof}[\textbf{Proof for Lemma \ref{bound_y_n}}]
From the definition of $y_n$, we have
\begin{eqnarray*}
y_{n+1}&=&y_n+a(n)\big(\big(I+B^{-1}A\big)y_n \\
&&\;\;\;\;\;\;\;\;\;\;\;\;\;\;\;\;\;\;-y_n+B^{-1}b\big) \\
&=&\left(1-a(n)\right)y_n \\ &&\;\;\;+a(n)\left(\left(I+B^{-1}A\right)y_n+B^{-1}b\right).
\end{eqnarray*}
Here
\begin{eqnarray*}
\|y_{n+1}\|_H&\leq& \left(1-a(n)\right)\|y_n\|_H+a(n)\beta\|y_n\|_H+a(n)\|B^{-1}b\|_H \\
&=& (1-(1-\beta)a(n))\|y_n\|_H+a(n)\|B^{-1}b\|_H.
\end{eqnarray*}
For $n,m\geq n_0$, define $\psi(n,m)\coloneqq\prod_{k=m}^{n-1}(1-(1-\beta)a(k))$ if $n>m$ and $1$ otherwise. Note that, since $a(k)<1$, $0\leq\psi(n,m)\leq1$ for all $n,m\geq N$. Then
\begin{equation*}
\|y_{n+1}\|_H-\frac{\|B^{-1}b\|_H}{1-\beta}\leq (1-(1-\beta)a(n))\left(\|y_n\|_H-\frac{\|B^{-1}b\|_H}{1-\beta}\right).
\end{equation*}
Now $\|y_{n_0}\|_H\leq\|y_{n_0}\|_H+\frac{\|B^{-1}b\|_H}{1-\beta}=\psi(n_0,n_0)\|y_{n_0}\|_H+\frac{\|B^{-1}b\|_H}{1-\beta}$. Suppose
\begin{equation}\label{induction_step}
    \|y_n\|_H\leq\psi(n,n_0)\|y_{n_0}\|_H+\frac{\|B^{-1}b\|_H}{1-\alpha}
\end{equation}
for some $n\geq n_0$. Then,
\begin{align*}
\|y_{n+1}\|_H-\frac{\|B^{-1}b\|_H}{1-\beta}&\leq(1-(1-\beta)a(n))\Bigg(\psi(n,n_0)\|y_{n_0}\|_H+\frac{\|B^{-1}b\|_H}{1-\beta}-\frac{\|B^{-1}b\|_H}{1-\beta}\Bigg)\nonumber\\
&\leq\psi(n+1,n_0)\|y_{n_0}\|
\end{align*}
By induction, (\ref{induction_step}) holds for all $n\geq n_0$, which completes the proof of Lemma \ref{bound_y_n}. 
\end{proof}
We next prove the probabilistic bound on $\|b_n-b\|_H$. 
\begin{proof}[\textbf{Proof for Lemma \ref{lemma:conc_ineq} (a)}]
By triangle inequality, 
$$\|b_n-b\|_H\leq \|b_n-\mathbb{E}(b_n)\|_H + \|\mathbb{E}(b_n)-b\|_H.$$
The second term here is bounded by $\frac{C_1}{n+1}$ by Lemma 4.3 (d) from \cite{LSPEBert} where $C_1$ is a positive constant. We shall use Theorem \ref{thm:Paulin} in the Appendix to obtain a probabilistic bound on the first term. We know that 
$$b_n=\frac{1}{n+1}\sum_{m=0}^n\sum_{t=0}^m(\alpha\lambda)^{m-t}\varphi(X_t)k(X_m).$$
Let $f_j:S^{n+1}\rightarrow\mathbb{R}$ be defined as:
$$f_j(x_0,\cdots,x_n)\coloneqq\sum_{m=0}^n\sum_{t=0}^m(\alpha\lambda)^{m-t}\varphi_j(x_t)k(x_m).$$
Let $x=(x_0,\cdots,x_{i-1}, x_i, x_{i+1}, \cdots, x_n)$ and\\ $y=(x_0,\cdots,x_{i-1}, y_i, x_{i+1}, \cdots, x_n)$. Then,
\begin{align*}
 f_j(x) - f_j(y) =& \Big((\alpha\lambda)^i\varphi_j(x_0)+\cdots+(\alpha\lambda)\varphi_j(x_{i-1})\Big)\Big(k(x_i)-k(y_i)\Big)\\
 &+ \Big(\varphi_j(x_i)k(x_i)-\varphi_j(y_i)k(y_i)\Big)+\Big(\varphi_j(x_i)-\varphi_j(y_i)\Big)\Big((\alpha\lambda)k(x_{i+1})+\cdots+(\alpha\lambda)^{n-i}k(x_n)\Big).
\end{align*}
Then 
\begin{align*}
|f_j(x)-f_j(y)|&\leq 2\Phi_{max}k_{max}\Big((\alpha\lambda)\frac{1-(\alpha\lambda)^i}{1-(\alpha\lambda)}+1+(\alpha\lambda)\frac{1-(\alpha\lambda)^{n-i}}{1-(\alpha\lambda)}\Big)\\
&\leq 2\Phi_{max}k_{max}\left(1+2\frac{\alpha\lambda}{1-\alpha\lambda}\right)\\
&\eqqcolon& d_1.
\end{align*}
Taking $c=(c_0,\cdots,c_n)$ with $c_i=d_1 \ \forall i$, we can apply Theorem \ref{thm:Paulin} of the Appendix A to $f_j(\cdot)$ with $X=(X_0,\cdots, X_n)^T$. Here $\|c\|_2^2=(n+1)d_1^2$ and $\tau_{min}$ is as defined in Theorem \ref{thm:Paulin}. 
$$    \mathbb{P}(|f_j(X)-\mathbb{E}(f_j(X))|\geq t)\leq2\exp\left(\frac{-2t^2}{(n+1)d_1^2\tau_{min}}\right).$$ 
Substituting $t=(n+1)a_1$, we get
\begin{align*}
\mathbb{P}\Big(|f_j(X)-\mathbb{E}(f_j(X))|\geq (n+1)a_1 \Big)\leq
2\exp\left(\frac{-2(n+1)^2a_1^2}{(n+1)d_1^2\tau_{min}}\right)
\end{align*}
\begin{align*}
\implies \mathbb{P}\Big(\Big|\frac{1}{n+1}f_j(X)-\mathbb{E}\left(\frac{1}{n+1}f_j(X)\right)\Big|\geq a_1\Big)\leq
2\exp\left(\frac{-2(n+1)a_1^2}{d_1^2\tau_{min}}\right) 
\end{align*}
\begin{align*}
&\implies \mathbb{P}\Big(|b_{n_j}-\mathbb{E}(b_{n_j})|\geq a_1 \Big)\leq2\exp\left(\frac{-2(n+1)a_1^2}{d_1^2\tau_{min}}\right).
\end{align*}
Here $b_{n_j}$ denotes the $j$\textsuperscript{th} element of $b_n$. Using union bound, we now have
\begin{align*}
\mathbb{P}\Big(|b_{n_j}-\mathbb{E}(b_{n_j})|\leq a_1\ \forall j\in\{1,\cdots,M\}\Big)\geq1-2M\exp\left(\frac{-2(n+1)a_1^2}{d_1^2\tau_{min}}\right) 
\end{align*}
\begin{align*}
\implies \mathbb{P}\Big(\|b_n-\mathbb{E}(b_n)\|_H\leq \sqrt{\lambda_{max}(H)M}a_1 \Big)\geq
1-2M\exp\left(\frac{-2(n+1)a_1^2}{d_1^2\tau_{min}}\right)
\end{align*}
\begin{align*}
\implies \mathbb{P}\Big(\|b_n-\mathbb{E}(b_n)\|_H\leq a \Big)\geq
1-2M\exp\left(\frac{-2(n+1)a^2}{M\lambda_{max}(H)d_1^2\tau_{min}}\right).
\end{align*}
The second statement  is obtained by applying Lemma \ref{lemma:vec-mat-bounds} (a) to $b_n-\mathbb{E}(b_n).$ Define $\gamma_1\coloneqq\lambda_{max}(H)d_1^2\tau_{min}/2$. Then, 
$$\mathbb{P}\Big(\|b_n-\mathbb{E}(b_n)\|_H\leq a\Big)\geq 1-2M\exp\left(\frac{-(n+1)a^2}{M\gamma_1}\right),$$
which completes the proof for Lemma \ref{lemma:conc_ineq} (a).
\end{proof}
Next we give the proof for Lemma \ref{lemma:conc_ineq} (b), i.e., the bound on $\|A_n-A\|_H.$
\begin{proof}[\textbf{Proof for Lemma \ref{lemma:conc_ineq} (b)}]
By triangle inequality, 
$$\|A_n-A\|_H\leq \|A_n-\mathbb{E}(A_n)\|_H + \|\mathbb{E}(A_n)-A\|_H.$$
The second term here is bounded by $\frac{C_2}{n+1}$ using Lemma 4.3 (d) from \cite{LSPEBert} where $C_2$ is a positive constant. Recall that
$$A_n=\frac{1}{n+1}\sum_{m=0}^n\sum_{t=0}^m(\alpha\lambda)^{m-t}\varphi(X_t)(\alpha\varphi(X_{m+1})^T-\varphi(X_m)^T).$$

Let $f_{jl}:S^{n+2}\rightarrow\mathbb{R}$ be defined as:
\begin{align*}
f_{jl}(x_0,&\cdots,x_n,x_{n+1})\coloneqq\\
&\sum_{m=0}^n\sum_{t=0}^m(\alpha\lambda)^{m-t}\varphi_j(x_t)(\alpha\varphi_l(x_{m+1})-\varphi_l(x_m)).
\end{align*}
Let $x=(x_0,\cdots,x_{i-1}, x_i, x_{i+1}, \cdots, x_n,x_{n+1})$ and\\ $y=(x_0,\cdots,x_{i-1}, y_i, x_{i+1}, \cdots, x_n,x_{n+1})$. Then
\begin{align*}
f_{jl}(x)-f_{jl}(y)=&\Big((\alpha\lambda)^{i-1}\varphi_j(x_0)+\cdots+\varphi_j(x_{i-1})\Big)\Big(\alpha\varphi_l(x_i)-\alpha\varphi_l(y_i)\Big)\\
&+\Big(\alpha\varphi_j(x_i)-\alpha\varphi_j(y_i)\Big)\Big(\varphi_l(x_{i+1})+\cdots+(\alpha\lambda)^{n-i}\varphi_l(x_{n+1})\Big)\\
&-\Big((\alpha\lambda)^i\varphi_j(x_0)+\cdots+(\alpha\lambda)\varphi_j(x_{i-1})\Big)\Big(\varphi_l(x_i)-\varphi_l(y_i)\Big) \\
&- \Big(\varphi_j(x_i)\varphi_l(x_i)-\varphi_j(y_i)\varphi_l(y_i)\Big) \\
&-\Big(\varphi_j(x_i)-\varphi_j(y_i)\Big)\Big((\alpha\lambda)\varphi_l(x_{i+1})+\cdots+(\alpha\lambda)^{n-i}\varphi_l(x_n)\Big).
\end{align*}
Then, 
\begin{align*}
 |f_{jl}(x)-f_{jl}(y)| \leq \;& 2\Phi_{max}^2\bigg(\alpha\frac{1-(\alpha\lambda)^i}{1-(\alpha\lambda)} + \alpha\frac{1-(\alpha\lambda)^{n-i}}{1-(\alpha\lambda)} + (\alpha\lambda)\frac{1-(\alpha\lambda)^i}{1-(\alpha\lambda)} + 1+  (\alpha\lambda)\frac{1-(\alpha\lambda)^{n-i}}{1-(\alpha\lambda)}\bigg)\\
\leq\;& 2\Phi_{max}^2\left(1+2\frac{1+\alpha\lambda}{1-\alpha\lambda}\right) \ \eqqcolon \ d_3.
\end{align*}
Taking $c=(c_0,\cdots,c_{n+1})$ with $c_i=d_3 \ \forall i$, we can now apply Theorem \ref{thm:Paulin} to $f_{jl}(\cdot)$ with $X=(X_0,\cdots, X_{n+1})^T$. Here $\|c\|_2^2=(n+2)d_3^2$. 
\begin{align*}
    \mathbb{P}(|f_{jl}(X)-\mathbb{E}(f_{jl}(X))|\geq t)&\leq2\exp\left(\frac{-2t^2}{(n+2)d_3^2\tau_{min}}\right) \\
    &\leq2\exp\left(\frac{-2t^2}{(n+1)d_3^2\tau_{min}}\right). \\
\end{align*}    
As in the proof for the previous part, this implies
$$\mathbb{P}\Big(\|A_n-\mathbb{E}(A_n)\|_H\leq a \Big)\geq1-2M^2\exp\left(\frac{-(n+1)a^2}{M^2\gamma_2}\right).
$$
where $\gamma_2\coloneqq\lambda_{max}(H)d_3^2\tau_{min}/(2\lambda_{min}(H))$. This completes the proof for Lemma \ref{lemma:conc_ineq} (b).
\end{proof}

Continuing, we finally prove probabilistic bounds for $\|B_n-B\|_H$ and $\|B_n^{-1}-B^{-1}\|_H$. 
\begin{proof}[\textbf{Proof for Lemma \ref{lemma:conc_ineq} (c)}]
By triangle inequality, 
$$\|B_n-B\|_H\leq \|B_n-\mathbb{E}(B_n)\|_H + \|\mathbb{E}(B_n)-B\|_H.$$
Recall that 
$$B_n=\frac{1}{n+1}\sum_{m=0}^n\varphi(X_m)\varphi(X_m)^T.$$
Let $\kappa_n(i)$ denote the number of visits to state $i$ up to iteration $n$. Then 
\begin{align*}
  & B_n=\sum_{i\in S}\frac{\kappa_n(i)}{n+1}\varphi(i)\varphi(i)^T \\
   \implies & \mathbb{E}(B_n)=\sum_{i\in S}\mathbb{E}\left(\frac{\kappa_n(i)}{n+1}\right)\varphi(i)\varphi(i)^T.
\end{align*}
But 
$$B=\sum_{i\in S}\pi(i)\varphi(i)\varphi(i)^T,$$
and hence
$$\mathbb{E}(B_n)-B=\sum_{i\in S}\left(\mathbb{E}\left(\frac{\kappa_n(i)}{n+1}\right)-\pi(i)\right)\varphi(i)\varphi(i)^T.$$
Note that $\left|\mathbb{E}\left(\frac{\kappa_n(i)}{n+1}\right)-\pi(i)\right|$ is bounded by $\left(\frac{C_3}{n+1}\right)$ for some constant $C_3 > 0$  by Lemma 4.2 (c) from \cite{LSPEBert}. Combining this with Lemma \ref{lemma:vec-mat-bounds} (b) gives us:
$$\|\mathbb{E}(B_n)-B\|_H\leq \frac{C_3}{n+1}.$$
Now define $f_{jl}:S^{n+1}\rightarrow \mathbb{R}$ as:
$$f_{jl}(x_0,\cdots,x_n)\coloneqq\sum_{m=0}^{n}\varphi_j(x_m)\varphi_l(x_m)$$
Let $x=(x_0,\cdots,x_{i-1}, x_i, x_{i+1}, \cdots, x_n)$ and\\ $y=(x_0,\cdots,x_{i-1}, y_i, x_{i+1}, \cdots, x_n)$. Then
$$f_{jl}(x)-f_{jl}(y)=\varphi_j(x_i)\varphi_l(x_i)-\varphi_j(y_i)\varphi_l(y_i),$$
leading to
$$|f_{jl}(x)-f_{jl}(y)|\leq2\Phi_{max}^2.$$
Define $d_2\coloneqq2\Phi_{max}^2$, $c=(c_0,\cdots,c_n)$ with $c_i=d_2 \ \forall i$. We can  apply Theorem \ref{thm:Paulin} to $f_{jl}(\cdot)$ with $X=(X_0,\cdots, X_n)^T$ and $\|c\|_2^2=(n+1)d_2^2$ to get 
$$    \mathbb{P}(|f_{jl}(X)-\mathbb{E}(f_{jl}(X))|\geq t)\leq2\exp\left(\frac{-2t^2}{(n+1)d_2^2\tau_{min}}\right).$$ 
As in the proof of previous parts, this implies
$$ \mathbb{P}\Big(|B_{n_{jl}}-\mathbb{E}(B_{n_{jl}})|\geq a_1 \Big)\leq2\exp\left(\frac{-2(n+1)a_1^2}{d_2^2\tau_{min}}\right).$$
Using union bound and Lemma \ref{lemma:vec-mat-bounds} (b) for $B_n-\mathbb{E}(B_n)$, we obtain
\begin{align*}
\mathbb{P}\Big(|B_{n_{jl}}-\mathbb{E}(B_{n_{jl}})|\leq a_1\ \forall j,l\in\{1,\cdots,M\}\Big)\geq1-2M^2\exp\left(\frac{-2(n+1)a_1^2}{d_2^2\tau_{min}}\right) 
\end{align*}
\begin{align*}
    \implies \mathbb{P}\Big(\|B_n-\mathbb{E}(B_n)\|_H\leq \sqrt{\frac{\lambda_{max}(H)}{\lambda_{min}(H)}}Ma_1 \Big)\geq
    1-2M^2\exp\left(\frac{-2(n+1)a_1^2}{d_2^2\tau_{min}}\right)
\end{align*}
\begin{align*}
\implies \mathbb{P}\Big(\|B_n-&\mathbb{E}(B_n)\|_H\leq a \Big)\geq
1-2M^2\exp\left(\frac{-2\lambda_{min}(H)(n+1)a^2}{M^2\lambda_{max}(H)d_2^2\tau_{min}}\right).
\end{align*}
Define  $\gamma_3\coloneqq\lambda_{max}(H)d_2^2\tau_{min}/(2\lambda_{min}(H))$ to obtain
$$\mathbb{P}\Big(\|B_n-\mathbb{E}(B_n)\|_H\leq a \Big)\geq1-2M^2\exp\left(\frac{-(n+1)a^2}{M^2\gamma_3}\right).$$

This gives us a probabilistic bound on $\|B_n-B\|_H$. Now we use this to obtain the bound on $\|B_n^{-1}-B^{-1}\|_H$. For matrices $B_n$ and $B$, it can be seen that
$$
B_n^{-1}-B^{-1} = B_n^{-1}(B-B_n)B^{-1},
$$
which implies
$$
\|B_n^{-1}-B^{-1}\|_H \leq \|B_n^{-1}\|_H\|B-B_n\|_H\|B^{-1}\|_H.
$$
By Lemma 4.3 (a) from \cite{LSPEBert}, there exists a constant $C^{'}$ such that $\|B_n^{-1}\|_H \leq C^{'} \;\;\forall\;\; n$. Hence
$$ \|B_n^{-1}-B^{-1}\|_H \leq C^{'}\|B-B_n\|_H\|B^{-1}\|_H.
$$
Then,
\begin{align*}
    \mathbb{P}\Bigg(\|B_n-B\|_H \leq \frac{a}{C^{'}\|B^{-1}\|_H}+ \frac{C_3}{n+1} \Bigg)\geq
    1-2M^2\exp\left(\frac{-(n+1)a^2}{M^2(C^{'})^2\|B^{-1}\|_H^2\gamma_3}\right)
\end{align*}
\begin{align*}
\implies \mathbb{P}\Bigg(\|B_n^{-1}-B^{-1}\|_H\leq a+ \frac{C_4}{n+1} \Bigg)\geq
1-2M^2\exp\left(\frac{-(n+1)a^2}{M^2(C^{'})^2\|B^{-1}\|_H^2\gamma_3}\right).
\end{align*}
where $C_4 = C^{'}\|B^{-1}\|_HC_3$.
\end{proof}

The next lemma shows that $\|A_n\|_H$ is bounded for all $n$.
\begin{lemma}\label{lemma:A_n-bound}
There exists a constant $C_5$ such that $\|A_n\|_H \leq C_5$ for all $n$.
\end{lemma}
\begin{proof}
We have defined $z_m$ in (\ref{def:b_n-and-z_m}) as
$$ z_m = \sum_{t=0}^m (\alpha\lambda)^{m-t}\varphi(X_t).
$$
Let $m_1 = \max_{i \in S}\|\varphi(i)\|_H $. Then,
$$
\|z_m\|_{H} \leq \max_{i \in S}\|\varphi(i)\|_H \sum_{t=0}^m (\alpha\lambda)^{m-t} = m_1\frac{1-(\alpha\lambda)^m}{1-\alpha\lambda}.
$$
Therefore $\|z_m\|_{H} \leq \frac{m_1}{1-\alpha\lambda} \;\;\; \forall m $. Using this in (\ref{def:A_n-and-B_n}),
\begin{align*}
\|A_n\|_H \ &\leq \frac{1}{n+1}\sum_{m=0}^n\frac{m_1}{1-\alpha\lambda}\|(\alpha\varphi(X_{m+1})-\varphi(X_m))\|_{H}\nonumber\\
&\leq\frac{1}{n+1}\frac{m_1}{1-\alpha\lambda}\sum_{m=0}^n(\alpha+1)\max_{i \in S}\|\varphi(i)\|_H\nonumber\\
&\leq  m_1^2\frac{(1+ \alpha)}{1-\alpha\lambda} \eqqcolon C_5.
\end{align*}
\end{proof}

Next we present the proof for the bound on $\delta(n)$.
\begin{proof}[\textbf{Proof for Lemma \ref{lemma:eps-delta} (a)}]
We have
\begin{align*}
\|\delta(n)\|_{H} &\leq
\|(B_n^{-1}-B^{-1})A_n \|_{H} + \|B^{-1}(A_n-A)\|_{H}\nonumber\\
&\leq \|B_n^{-1}-B^{-1}\|_{H}\|A_n \|_{H} + \|B^{-1}\|_{H}\|A_n-A\|_{H}\nonumber\\
&\leq \|B_n^{-1}-B^{-1}\|_{H}C_5+ \|B^{-1}\|_{H}\|A_n-A\|_{H}.\nonumber
\end{align*}
By using part (c) of Lemma \ref{lemma:conc_ineq},
\begin{align*}
\mathbb{P}\Big(\|B_n^{-1}-B^{-1}\|_H \leq \frac{\delta}{2C_5}+ \frac{C_4}{n+1} \Big)\geq
1-2M^2\exp\left(\frac{-(n+1)\delta^2}{4M^2C_5^2(C^{'})^2\|B^{-1}\|_H^2\gamma_3}\right)\nonumber
\end{align*}
\newline
Using part (b) of Lemma \ref{lemma:conc_ineq},
\begin{align*}
\mathbb{P}\Big(\|A_n-A\|_H\leq \frac{\delta}{2\|B^{-1}\|_H} + \frac{C_2}{n+1} \Big)\geq
 1-2M^2\exp\left(\frac{-(n+1)\delta^2}{4M^2\|B^{-1}\|_H^2\gamma_2}\right).
\end{align*}
Combining the above inequalities using the union bound,
\begin{align*}
\mathbb{P}\Big(\|\delta(n)\|_H\leq \delta+ \frac{K_1}{(n+1)}\Big)\ &\geq  1  -2M^2\exp\left(\frac{-(n+1)\delta^2}{4M^2C_5^2(C^{'})^2\|B^{-1}\|_H^2\gamma_3}\right)-2M^2\exp\left(\frac{-(n+1)\delta^2}{4M^2\|B^{-1}\|_H^2\gamma_2}\right)\nonumber\\
&\geq 1 - 4M^2\exp\left(\frac{-(n+1)\delta^2}{M^2K_2}\right),\nonumber
\end{align*}
where $K_1 \coloneqq C_5C_4 + \|B^{-1}\|_HC_2$ and 
$K_2\coloneqq\max\left(4C_5^2(C^{'})^2\|B^{-1}\|_H^2\gamma_3,4\|B^{-1}\|_H^2\gamma_2\right).$
\end{proof}
Now we prove the bound on $\varepsilon(n)$.
\begin{proof}[\textbf{Proof for Lemma \ref{lemma:eps-delta} (b)}]
Recall that 
$$
\varepsilon_1(n)\coloneqq B_n^{-1}\left((A_n-A)\right)+ \left(B_n^{-1}-B^{-1}\right)A,
$$
and 
$$
\varepsilon_2(n)\coloneqq B_n^{-1}(b_n-b)+\left(B_n^{-1}-B^{-1}\right)b.
$$
Then using Lemma \ref{lemma:conc_ineq} (b) to bound $\|A_n-A\|_H$, and Lemma \ref{lemma:conc_ineq} (c) to bound $\|B_n^{-1}-B^{-1}\|_H$, we get
\begin{equation*}
P\left(\|\epsilon_1(n)\|_H\leq a+\frac{K_5}{n+1}\right) \geq 1-4M^2\exp\left(\frac{-(n+1)a^2}{M^2K_7}\right)
\end{equation*}
where $K_5=C_2C'+C_4\|A\|_H$ and 
$$K_7=\max\left\{4C'^2\gamma_2, 4C'^2\|B^{-1}\|_H^2\gamma_3\|A\|_H^2\right\}.$$ Similarly, using Lemma \ref{lemma:conc_ineq} (a) to bound $\|b_n-b\|_H$ and Lemma \ref{lemma:conc_ineq} (c) to bound $\|B_n^{-1}-B^{-1}\|_H$, we get
\begin{equation*}
P\left(\|\epsilon_2(n)\|_H\leq a+\frac{K_6}{n+1}\right) \geq 1-4M^2\exp\left(\frac{-(n+1)a^2}{M^2K_8}\right)
\end{equation*}
where $K_6=C_1C'+C_4\|b\|_H$ and 
$$K_8=\max\left\{4C'^2\gamma_1,4C'^2\|B^{-1}\|_H^2\gamma_3\|b\|_H^2 \right\}.$$ 
These imply that 
$$\|\epsilon_1(n)\|_H\leq a+\frac{K_5}{n+1}\; \textrm{and}\;\|\epsilon_2(n)\|_H\leq a+\frac{K_6}{n+1}$$ together hold with probability greater than $$1-8M^2\exp\left(\frac{-(n+1)a^2}{M^2K_4}\right),$$ where $K_4=\max\{K_7, K_8\}$. Substituting $a=\epsilon/(n+1)^{1/2-\theta}$ completes the proof.
\end{proof}

We now give a proof for the final part of this lemma.
\begin{proof}[\textbf{Proof of Lemma \ref{lemma:eps-delta} (c)}]
Note that Lemma \ref{lemma:eps-delta} (b) implies that $$\|\varepsilon(k)\|_H \leq K_3'\left(\frac{\epsilon}{(k+1)^{(1/2-\theta)}} + \frac{1}{k+1}\right)(\|r_{n_0}\|_H+1)$$ holds with probability greater than $$1-8M^2\exp\left(\frac{-(k+1)^{2\theta}\epsilon^2}{M^2K_4}\right),$$ where 
\begin{equation*}
K_3' =\max\Big\{2,2\|B^{-1}b\|_H/(1-\beta), K_5,K_5\|B^{-1}b\|_H/(1-\beta)+K_6\Big\}.
\end{equation*} 
Then under the assumption on stepsize $\{a(n)\}$, we have
\begin{eqnarray*}
\chi(m,k+1)&=&\prod_{\ell=k+1}^{\ell=m} (1-a(\ell))\leq \exp\left(-\sum_{\ell=k+1}^m a(\ell)\right)\\
&\leq& \exp\left(-\sum_{\ell=k+1}^m \frac{\mu_1}{\ell}\right)\leq \left(\frac{k+1}{m}\right)^{\mu_1}.
\end{eqnarray*}
Thus
\begin{align*}
\sum_{k=n_0}^m \frac{\chi(m,k+1)a(k)}{(k+1)^{1/2-\theta}}&\leq \mu_3\sum_{k=n_0}^m \left(\frac{k+1}{m}\right)^{\mu_1}\left(\frac{1}{k}\right)^{\mu_2}\frac{1}{(k+1)^{1/2-\theta}}\\
&\leq K_3^*m^{1/2+\theta-\mu_2}, 
\end{align*}
where $K_3^*$ is some positive constant. Similarly, for some positive $K_3^{**}$,
$$\sum_{k=n_0}^m \frac{\chi(m,k+1)a(k)}{k+1}\leq K_3^{**}m^{-\mu_2}.$$
Then, for $K_3=K_3'\max\{K_3^*,K_3^{**}\}$,
\begin{align*}
\sum_{k=n_0}^m\chi(m,k+1)a(k)\|\varepsilon(k)\|_H \leq K_3\left(\|r_{n_0}\|_H+1\right)\left(\epsilon m^{1/2+\theta-\mu_2}+m^{-\mu_2}\right),
\end{align*}
holds with probability greater than $$1-8M^2\sum_{k=n_0}^m\exp\left(\frac{-(k+1)^{2\theta}\epsilon^2}{M^2K_4}\right).$$ 
\end{proof}
}

\end{document}